\documentclass[journal]{IEEEtran}
\ifCLASSINFOpdf
  \usepackage[pdftex]{graphicx}
\else
\fi
\usepackage{amsmath}
\usepackage{amsthm}
\usepackage{amsfonts}
\usepackage{enumitem}
\usepackage{dcolumn}
\usepackage[ruled, vlined]{algorithm2e}
\usepackage{float}

\newtheorem{lemma}{Lemma}
\renewcommand{\vec}[1]{\boldsymbol{#1}}

\begin{document}

%
\title{Feature Selection Based on a Sparse Neural Network Layer with Normalizing Constraints}
%
%

%

\author{Peter Bugata and
Peter~Drot\'ar,~\IEEEmembership{Member,~IEEE}
\thanks{P. Bugata and P. Drot\'ar are with the Department of Computers and Informatics, Technical University of Ko\v{s}ice, Letn\'a 9, 42001 Ko\v{s}ice, Slovakia.}
\thanks{Manuscript received xxxx, xxxx; revised xxxx, xxxx. }}

%
%

\markboth{Journal of \LaTeX\ Class Files,~Vol.~XX, No.~X, XXXX XXXX}%
{Shell \MakeLowercase{\textit{et al.}}: Bare Demo of IEEEtran.cls for IEEE Journals}
%



\maketitle

\begin{abstract}
Feature selection (FS) is an important step in machine learning since it has been shown to improve prediction accuracy while suppressing the curse of dimensionality of high-dimensional data. Neural networks have experienced tremendous success in solving many nonlinear learning problems. Here, we propose a~new neural network-based FS approach that introduces two constraints, the satisfaction of which leads to a~sparse FS layer. We performed extensive experiments on synthetic and real-world data to evaluate the performance of our proposed FS method. In the experiments, we focus on high-dimensional, low-sample-size data since they represent the main challenge for FS. The results confirm that the proposed FS method based on a sparse neural network layer with normalizing constraints (SNeL-FS) is able to select the important features and yields superior performance compared to other conventional FS methods.

\end{abstract}


%
\IEEEpeerreviewmaketitle

\section{Introduction}
%
%
%
%
%
%
%
\IEEEPARstart{I}{n} recent years, a rapid increase in the amount of data has been observed. Data generated in areas such as healthcare, bioinformatics, transportation, social media, and online education are often high-dimensional and present challenges not only for effective and efficient data management but also in the application of data mining and machine learning techniques~\cite{li2017C}. In addition to the increased demands on computing resources, these high-dimensional data are tightly associated with the curse of dimensionality, which is a phenomenon that adversely affects machine learning algorithms designed for low-dimensional space~\cite{hastie2001}.

One approach to cope with the curse of dimensionality and related issues is feature selection (FS). This is the process of selecting a subset of features from the original feature set. The goal is to select the relevant features and to drop irrelevant, noisy, and redundant features. 
To date, many FS methods based on different principles, such as rough set theory~\cite{wang2020}, the graph-guided regularization approach~\cite{liu2020}, logistic regression with cardinality constraints~\cite{adeli2020}, and many others~\cite{li2017}, \cite{nie}, \cite{wangfs}, have been proposed.

Recently, artificial neural networks and deep learning have provided unprecedented performance on many nonlinear learning problems, such as pattern recognition, sequence recognition, system identification, and medical diagnoses.
Neural networks construct new and more abstract high-level features from the original low-level input variables and discover the distributed representations of data. The achievements of neural networks and deep learning have led to the opinion that advanced machine learning can be done without FS. However, in some cases, mainly where the number of observations is not sufficient, deep learning should be combined with FS to obtain better learning performance \cite{li2017B}. In general, irrelevant features can slow network training and increase requests for computational resources. FS decreases model complexity at the input level and is also helpful in maintaining the interpretability of the original features and understanding the trained model~\cite{cai2018}. 

On the other hand, neural networks and deep learning can be the basis for FS. A deep-learning-based FS approach is presented in~\cite{roy2015}, where it is used in the context of action recognition from video records and utilizes the contributions of input variables to the activation potentials of the first hidden layer. Another FS method~\cite{zou2015}, which is used for remote sensing scene recognition, selects important features according to the minimal reconstruction error in the deep belief network.
Several works~\cite{wasserman2015}, \cite{ibrahim2014}, \cite{zhang2020} introduced FS methods based on neural networks and applied their networks in bioinformatics.

Some neural network-based FS methods utilize the assumption that when training a neural network, the weights of useless features tend to be zero~\cite{cai2018}. To determine the importance of features for making a correct prediction, they define various saliency metrics, including the network weights, the derivative of the network loss function, or both~\cite{belue1995}. 
The saliency values computed after training the network with the full set of features determine the feature importance (scores, weights). 

Other algorithms sequentially eliminate irrelevant features according to some criteria evaluated in each step for all remaining features. The criterion used in~\cite{romero2008} is the difference between the values of the objective function before and after the removal of the feature. The~method in~\cite{setiono1997} eliminates useless features based on the prediction accuracy. The same criterion is used by an FS method~\cite{onnia2001}, which sequentially builds a~set of selected features starting with an empty set.

Recently, more attention has been paid to promoting the sparsity of deep neural networks. The inclusion of a sparse regularization term into the learning model results in zeroing out the redundant weights during the training process, and a~smaller number of nonzero weights leads to sparse feature scores.
Sparse regularizers are often based on the norms of the network weights. An~FS method introduced in~\cite{sun2017} selects the important input variables of a feedforward neural network with 
\(\ell_1\) regularization. The~method in~\cite{li2017D} utilizes \(\ell_1\) and \(\ell_{1/2}\) regularizations for eliminating the redundant dimensions of input data to compress the input layer of the network.
In an FS method~\cite{wasserman2015}, a sparse one-to-one linear layer is added between the input layer and the first hidden layer, and its weights determine the importance values of the input variables. The network is trained with elastic-net regularization~\cite{zou2005}, which is a~convex combination of \(\ell_1\) and \(\ell_2\) regularizations.

In this paper, we propose a new neural network-based FS method called sparse neural network layer FS (SNeL-FS). 
It utilizes a~special sparse layer of the feedforward neural network to select relevant features. Unlike previous work, the sparsity of this FS layer is achieved in a new way by a combination of two normalizing constraints. The feature importance values are then determined with saliency measures designed specifically for the FS layer. 
The proposed FS method is suitable for classification and regression tasks. The main contributions of this paper are summarized as follows:

\begin{enumerate}[label=\arabic*)]
\item We propose a novel FS method based on feedforward neural networks, which selects the most important features with a~special network layer suitable for FS added between the input layer and the first hidden layer.
\item To obtain the sparsity of the FS layer, two constraints are introduced. The first limits the weights of inputs, and the second regulates the standard deviations of outputs of the FS layer neurons. The constraints are dynamically balanced during network training.
\item Two saliency measures based on the FS layer weights are defined to determine the feature importance.
\item  The experimental results compared to the results of popular FS methods demonstrate the ability of the proposed method to identify relevant features even in high-dimensional data with small sample sizes.
\end{enumerate}

The rest of the paper is organized as follows:
Section \ref{sec:nnfs} describes the main idea and the theoretical concept of the proposed FS method. Section~\ref{sec:impl} focuses on its implementation aspects.
In Section \ref{sec:exp}, we first evaluate the performance of the novel FS method on artificial data and then compare its influence on the classification performance on real-world datasets. 
Finally, in Sections \ref{discus} and \ref{concl}, we discuss future research directions and present our conclusions.

\section{Proposed Method}
\label{sec:nnfs}
Let \(\mathcal{F} = \{X_1, X_2, \ldots, X_m\}\) be a set of \(m\) features (variables) and \(\{x_1, x_2, \ldots, x_n\}\) be a set of \(n\) observations of a dataset \(\vec{X} \in \mathbb{R}^{n \times m}\). Let \(\vec{Y} = (y_1, y_2, \ldots, y_n)\) be the target variable. In general, the goal of supervised FS is to find a~set \(\mathcal{S} \subset \mathcal{F}\) of \textit{dim} features, \(\textit{dim} \ll m\), that optimally characterizes the target~\(\vec{Y}\). Based on the set \(\mathcal{S}\), a new dataset \(\vec{X^*} \in \mathbb{R}^{n \times \textit{dim}}\) is extracted and keeps most of the information about \(\vec{X}\). The proposed FS method assigns scores to individual features according to the weights of the added FS layer, and the features with the highest scores form the set \(\mathcal{S}\).

\subsection{Feature Selection Layer} \label{subsec:motiv}
Consider a feedforward neural network that solves a given classification or regression task. 
Then, we include a special hidden dense layer between the input layer and the first hidden layer of this network for the purpose of FS. We denote this layer as the FS layer.

The FS layer has \textit{dim} neurons, where \textit{dim} is the number of input variables (features) to select.
No nonlinear activation function is used in this layer; thus, the layer represents a linear transformation of the input variables. The FS layer neurons do not use any threshold value (bias = 0), so the parameters of the FS layer are only the weights of the connections between the input layer and the FS layer. The FS layer is then fully connected to the first hidden layer of the original network.

The aim of the constructed FS layer in selecting important input variables is to find an optimal solution that simultaneously satisfies the following two conditions:
\begin{enumerate}[label=\roman*)]
\item The weights between the input layer and the FS layer take only the values $1$ or $0$, where $1$ means that the corresponding input variable is selected by the respective FS layer neuron and $0$ means the opposite.
\item For each FS layer neuron, the sum of the weights of the connections entering this neuron equals $1$.
\end{enumerate}

When these conditions are satisfied, each FS layer neuron selects exactly one input variable. Thus, the FS layer weights can be interpreted as variable selection.
Note that two FS layer neurons can select the same variable; hence, after eliminating the zero-weighted variables, the original space of the variables is transformed into a space with at most \textit{dim} dimensions.

While condition ii) limits the size of the weights, condition i) requires weight values of only $0$ or $1$, which seems to be critical because this condition cannot be obtained by continuous methods in neural networks. We approximate this constraint by the sparsity of the FS layer with an idea loosely inspired by batch normalization; a neural network technique that accelerates and to some extent separates the training of individual layers \cite{ioffe2015}.

Assume that the input dataset is standardized. Each input variable is standardized independently; i.e., each variable has a mean of~$0$ and a standard deviation of~$1$.
Then, the new features obtained as FS layer outputs also have a mean value of $0$. Consider some neuron of the FS layer. If only one connection with a weight of~$1$ enters this neuron and the weights of the other connections are~$0$, then the corresponding new feature is only a copy of the selected original variable, and its standard deviation is also~$1$. If several connections have nonzero weights smaller than~$1$ and with a sum of $1$, the standard deviation will decrease unless the respective variables are completely correlated.
Therefore, to achieve a~sparse solution, we replace condition i) with a new condition in which the standard deviation of the output of each FS layer neuron must be at least~$1$.

\subsubsection{FS Layer Definition}
Consider the dataset \(\vec{X}\) with the variables \(X_1, X_2, \ldots, X_m\). We assume that \(\vec{X}\) is standardized; i.e., the mean or expected value \(\operatorname{E}(X_j) = 0\), and the standard deviation and variance \(\operatorname{Var}(X_j) = 1\) for every \(j = 1,\ldots, m\). Next, consider the feedforward neural network to solve the corresponding classification or regression problem. The proposed SNeL-FS method modifies the network by including a dense FS layer of \textit{dim} neurons between the input layer and the first hidden layer.
In the FS layer, no nonlinear activation function and no bias are used; thus, the outputs, i.e.,
the activation values, of its neurons represent new variables and are defined as
\begin{equation*}
A_k =  \sum_{j=1}^m w_{jk} X_j, \quad \text{for } k = 1,\ldots, \textit{dim}, 
\end{equation*}
where \(w_{jk}\) denotes the weight of the connection between the \(j\)-th input neuron and the \(k\)-th FS layer neuron.

The presented FS method selects \textit{dim} important input variables of the dataset \(\vec{X}\) according to the optimal FS layer weights obtained by learning the modified neural network.
The network is trained to minimize a~given objective function, while the weights and activation values of the FS layer should satisfy the following two sets of conditions:
\begin{equation}
\label{_fs_1}
\sum_{j=1}^m \left|w_{jk}\right| \leq 1, \quad \text{for } k = 1,\ldots, \textit{dim}, 
\end{equation}
\begin{equation}
\label{_fs_2}
\operatorname{Var}\left(A_k\right) \geq 1, \quad \text{for } k = 1,\ldots, \textit{dim}. 
\end{equation}
The \(\operatorname{Var}(A_k)\) value is the variance of \(A_k\) computed over all observations of the input dataset \(\vec{X}\) or over observations included in each minibatch in the training process when used.

An~example of a~neural network with the~added FS layer that satisfies conditions (\ref{_fs_1}) and (\ref{_fs_2}) is depicted in Fig.~\ref{fig:NN_FS}. As proven below, each FS layer neuron selects one input variable, which is shown by highlighted connections.

\begin{figure}[ht]
	\centering
	\includegraphics[width=1.0\columnwidth]{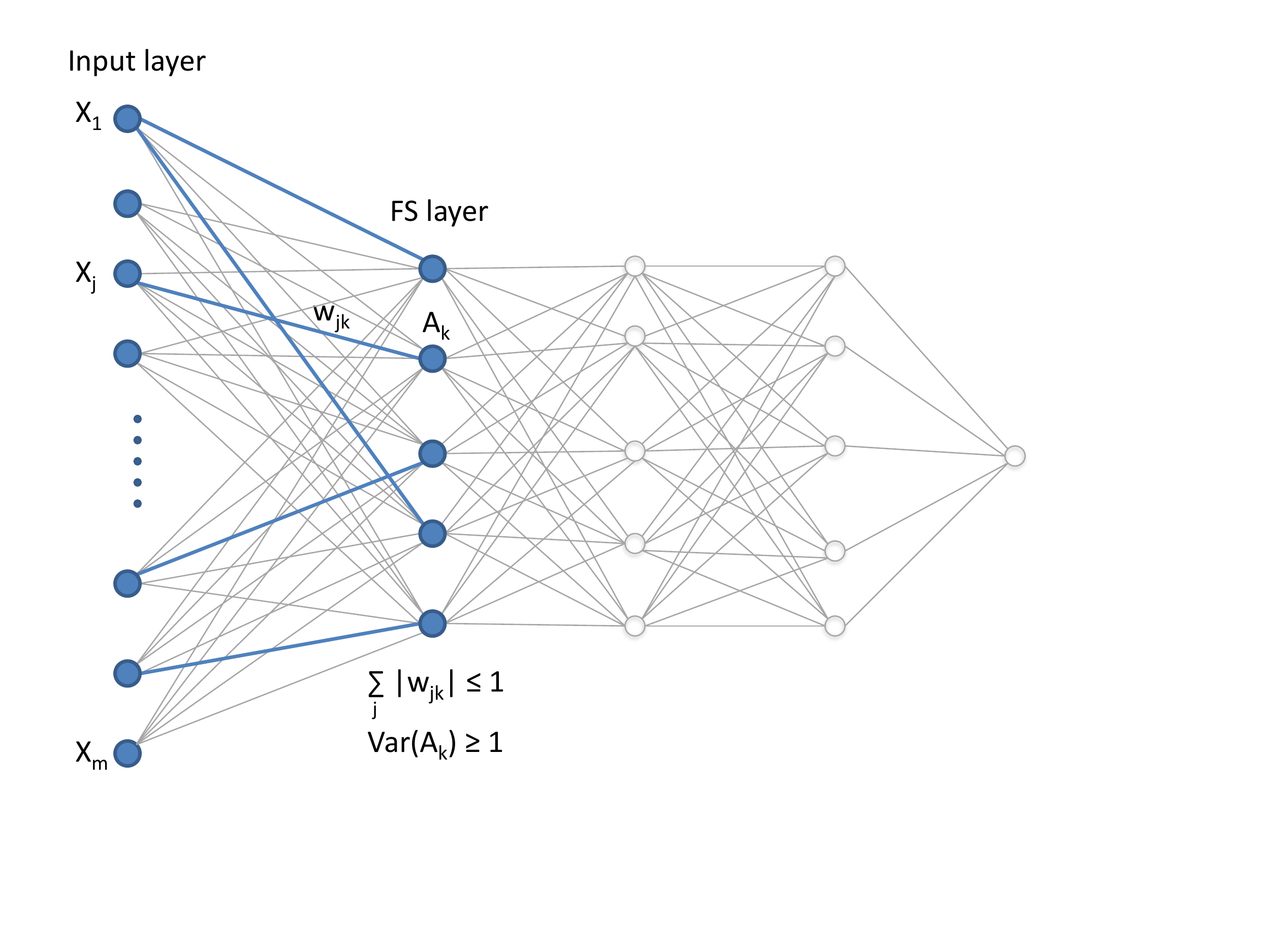}
\caption{Feature selection layer based method}
	\label{fig:NN_FS}
\end{figure}

\subsubsection{Relationship Among the FS Layer Constraints}
\label{subsec:nnfs_back}
When analyzing requirements (\ref{_fs_1}) and (\ref{_fs_2}) for the \(k\)-th neuron of the FS layer, it can be seen that they work against each other. This follows from the relationship between the sum of the absolute values of the weights entering the \(k\)-th neuron and the variance of its activation value \(A_k\).

\begin{lemma}
\label{lemma1}
Let \(\{X_1, X_2, \ldots, X_m\}\) be a set of \(m\) variables from a standardized dataset \(\vec{X}\). 
If \(A_k =  \sum_{j=1}^m w_{jk} X_j\), where \(k \in \{1, \ldots, \textit{dim}\}\) and \(w_{jk} \in \mathbb{R}\), for \(j \in \{1,\ldots, m\}\), then
\begin{equation}
\label{_fs_3}
\operatorname{Var}\left(A_k\right) \leq \left( \sum_{j=1}^m \left| w_{jk} \right| \right)^2.
\end{equation}
\end{lemma}

\begin{proof}
With the linearity of the expectation, the definitions of variance and covariance lead to the following expression of the variance of \(A_k\):
\begin{equation}
\label{_fs_4}
\operatorname{Var}\left(A_k\right) =
\sum_{i=1}^m \sum_{j=1}^m w_{ik} w_{jk} \operatorname{Cov}(X_i,X_j),
\end{equation}
where \(\operatorname{Cov}(X_i,X_j)\) is the covariance of the variables \(X_i\), \(X_j\). Because the dataset \(\vec{X}\) is standardized, the covariance values of \(X_i\) and \(X_j\) match with their correlation values \(\operatorname{Corr}(X_i,X_j)\), where the values range from -1 to 1.
Based on the properties of absolute values, the variance of \(A_k\) is bounded from above by the following sums:
\begin{equation*}
\begin{split}
&\operatorname{Var}\left(A_k\right) =
\left| \sum_{i=1}^m \sum_{j=1}^m w_{ik} w_{jk} \operatorname{Corr}(X_i,X_j) \right| \leq\\
& \sum_{i=1}^m \left| w_{ik} \right| \sum_{j=1}^m \left| w_{jk} \right| \left| \operatorname{Corr}(X_i,X_j) \right| \leq \sum_{i=1}^m \left| w_{ik} \right| \sum_{j=1}^m \left| w_{jk} \right|.
\end{split}
\end{equation*}
From this we get inequality~(\ref{_fs_3}).
\end{proof}

Assume condition (\ref{_fs_1}) holds for the \(k\)-th neuron of the FS layer. Then, inequality (\ref{_fs_3}) implies that the variance of \(A_k\) is bounded from above by one; i.e., \(\operatorname{Var}\left(A_k\right) \leq 1\). It is an inequality that is opposite the one in condition (\ref{_fs_2}) for the \mbox{\(k\)-th} neuron.
Conversely, if we assume that condition (\ref{_fs_2}) holds for the \mbox{\(k\)-th} neuron, then inequality (\ref{_fs_3}) implies an inequality that is opposite the one in condition (\ref{_fs_1}); i.e., \(\sum_{j=1}^m \left|w_{jk}\right| \geq 1\).
We have shown that conditions (\ref{_fs_1}) and (\ref{_fs_2}) work against each other. In addition, if they hold simultaneously for the \(k\)-th neuron, then  \(\sum_{j=1}^m \left|w_{jk}\right| = 1\) and \(\operatorname{Var}\left(A_k\right) = 1\).

Therefore, it follows from Lemma~\ref{lemma1} that if the input dataset is standardized and conditions (\ref{_fs_1}) and (\ref{_fs_2}) hold, then the outputs of the FS layer are also standardized.

\subsubsection{Sparsity of the FS Layer}
\label{subsec:nnfs_sprs}
Additionally, we examine which weights of the FS layer satisfy inequalities (\ref{_fs_1}) and (\ref{_fs_2}). We show that assuming there is no pair of fully correlated input variables, the solution is sparse. This assumption is not limiting because if two input variables are fully correlated, then one of them can be omitted without a loss of information.

\begin{lemma}
\label{lemma2}
Let \(\{X_1, X_2, \ldots, X_m\}\) be a set of \(m\) variables of a standardized dataset \(\vec{X}\). Suppose that for every \(i, j \in \{1,\ldots, m\}\), \(i \neq j\), the variables \(X_i, X_j\) are not completely correlated; i.e., \(\left|\operatorname{Corr}(X_i, X_j) \right| < 1\). Then, the solution of inequalities (\ref{_fs_1}) and (\ref{_fs_2}) for \(k \in \{1, \ldots, \textit{dim}\}\) is a~set of all vectors \(\vec{w_k} = (w_{1k}, w_{2k}, \ldots, w_{mk})\) that satisfy the following conditions:
%
\begin{enumerate}[label=\alph*)]
\item There is exactly one \(i \in \{1, \ldots, m\}\) such that \(\left |w_{ik} \right | = 1\).
\item For every \(j \in \{1, \ldots, m\}\), \(j \neq i\), the weight  \(w_{jk} = 0\).
\end{enumerate}
\end{lemma}

\begin{proof}
It is easy to see that each vector \(\vec{w_k}\)  
with a~single nonzero component whose absolute value equals~$1$ satisfies the system of inequalities (\ref{_fs_1}) and (\ref{_fs_2}) for \(k \in \{1, \ldots, \textit{dim}\}\). We prove that only such vectors are solutions of this system.

Let \(k\) be an arbitrary number from the set \(\{1, \ldots, \textit{dim}\}\) and let the vector \(\vec{w_k}\) satisfy the system of inequalities (\ref{_fs_1}) and (\ref{_fs_2}). We know that the variance of the activation \(A_k\) of the \(k\)-th FS layer neuron is bounded from above:
\begin{equation}
\label{_fs_5}
\begin{split}
\operatorname{Var}\left(A_k\right) \leq
 \sum_{i=1}^m \sum_{j=1}^m  \left| w_{ik} \right|\left| w_{jk} \right|  \left|\operatorname{Corr}(X_i,X_j) \right|. 
\end{split}
\end{equation}
Generally, \(\left| \operatorname{Corr}(X_i, X_j) \right| \leq 1\) for every \(i, j \in \{1,\ldots, m\}\); thus, for each term on the right side of inequality~(\ref{_fs_5}), the following is valid:
\begin{equation}
\label{_fs_6}
\left| w_{ik} \right|\left| w_{jk} \right|  \left|\operatorname{Corr}(X_i,X_j) \right| \leq \left| w_{ik} \right| \left| w_{jk} \right|.
\end{equation}

Assume that there are two different nonzero components \(w_{pk}, w_{qk}\) of the vector \(\vec{w_k}\).
According to the assumption in Lemma~\ref{lemma2}, \(\left| \operatorname{Corr}(X_p, X_q) \right| < 1\) and for the corresponding term of~(\ref{_fs_5}), the following holds:
\begin{equation}
\label{_fs_7}
\left| w_{pk} \right|\left| w_{qk} \right|  \left|\operatorname{Corr}(X_p,X_q) \right| < \left| w_{pk} \right| \left| w_{qk} \right|.
\end{equation}
After summation, we obtain the following:
\begin{equation}
\label{_fs_8}
\begin{split}
\operatorname{Var}\left(A_k\right) 
< \sum_{i=1}^m \sum_{j=1}^m \left| w_{ik} \right|  \left| w_{jk} \right| = \left( \sum_{j=1}^m \left| w_{jk} \right| \right)^2.
\end{split}
\end{equation}
Because the vector \(\vec{w_k}\) satisfies condition (\ref{_fs_1}), inequality (\ref{_fs_8}) implies that \(\operatorname{Var}\left(A_k\right) < 1\). However, this is a contradiction to the assumption that the vector \(\vec{w_k}\) satisfies condition (\ref{_fs_2}), \(\operatorname{Var}\left(A_k\right) \geq 1\).

Therefore, if the vector \(\vec{w_k}\) satisfies conditions (\ref{_fs_1}) and (\ref{_fs_2}), it has at most one nonzero component. If all components are zero, then condition (\ref{_fs_2}) does not hold. Thus, the only possibility is that the vector \(\vec{w_k}\) contains exactly one nonzero component, and according to conditions (\ref{_fs_1}) and (\ref{_fs_2}), its absolute value must be equal to $1$.
\end{proof}

This lemma implies that for each FS layer neuron, the solutions of inequalities (\ref{_fs_1}) and (\ref{_fs_2}) are vectors of the weights entering this neuron, which are not only sparse but even have exactly one nonzero component. This means that the FS layer realizes the selection of the input variables. Ideally, each FS layer neuron selects exactly one input variable corresponding to the connection with the nonzero weight.

\subsection{Neural Network Model}
\label{subsec:nnfs_nn_train}
Let us consider a feedforward neural network that forms the basis for the use of the SNeL-FS method. Suppose there are \(H\) hidden layers in the model. We denote the model parameter by \( \vec{\theta} = \{\vec{W}^{1}, \vec{b}^{1}, \ldots, \vec{W}^{H+1}, \vec{b}^{H+1}\} \), where \(\vec{W}^{h}\) is the weight matrix of the connections between the (\(h-1\))-th and \(h\)-th layers and \(\vec{b}^{h}\) is the bias in the \(h\)-th layer for \(h = 1, \ldots, H+1\).
Let the minimized objective function be of the form
\begin{equation}
\label{_fs_10}
f(\vec{\theta}) = l(\vec{\theta}) + \lambda \sum_{h=1}^{H+1} \Omega(\vec{W}^{h}),
\end{equation}
which is the sum of the loss function \(l(\vec{\theta})\) and the regularization term \(\sum_{h=1}^{H+1} \Omega(\vec{W}^{h})\) multiplied by the regularization parameter \(\lambda \in R_0^+\). The type of loss function used depends on the problem being solved. The regularization term can be added to avoid the possible overfitting of the model.

%
After including the FS layer between the input layer and the first hidden layer, the model is slightly changed, and its new parameter is
\( \vec{\tilde{\theta}} = \{\vec{W}, \vec{\tilde{W}}^{1}, \vec{\tilde{b}}^{1}, \ldots, \vec{\tilde{W}}^{H+1}, \vec{\tilde{b}}^{H+1}\} \), where \(\vec{W} = (w_{jk})\) is the weight matrix connecting the input layer to the FS layer.
The original optimization task is transformed into a new task, the results of which are used to select the most important input variables. Because the added FS layer should satisfy conditions (\ref{_fs_1}) and (\ref{_fs_2}), the original task is changed to the following constrained optimization problem: 
\begin{equation*}
\min_{ \vec{\tilde{\theta}}} f( \vec{\tilde{\theta}}) = l( \vec{\tilde{\theta}})  + \lambda \sum_{h=1}^{H+1} \Omega(\vec{\tilde{W}}^{h}),\\[-0.2cm]
\end{equation*}
\begin{equation}
\label{_fs_11}
\sum_{j=1}^m \left|w_{jk}\right|- 1 \leq 0, \quad \text{for } k = 1,\ldots, \textit{dim},\\[-0.2cm]
\end{equation}
\begin{equation*}
1 - \operatorname{Var}\left(A_k\right) \leq 0, \quad \text{for } k = 1,\ldots, \textit{dim}.
\end{equation*}
Problem (\ref{_fs_11}) aims to minimize the objective function \(f(\vec{\tilde{\theta}})\) with respect to the parameter \(\vec{\tilde{\theta}}\) and subject to two sets of additional constraints. Because solving this problem is extremely difficult, we have transformed it to an unconstrained optimization problem, the solution of which approximates the solution of constrained problem (\ref{_fs_11}). The idea is based on the Karush-Kuhn-Tucker (KKT) approach~\cite{boyd2004}, which generalizes the method of Lagrange multipliers.
However, to avoid the inclusion of a large number of hyperparameters, we applied some simplification. Finally, the optimization problem is defined as
\begin{equation}
\label{_fs_13}
\begin{split}
\min_{\vec{\tilde{\theta}}} &F(\vec{\tilde{\theta}}) = f(\vec{\tilde{\theta}}) + \lambda_S \Omega_S(\vec{W}) + \lambda_A \Omega_A(\vec{W}),  \\
&\Omega_S(\vec{W}) = \sum_{k=1}^{\textit{dim}} \max (0, \sum_{j=1}^m \left|w_{jk}\right|- 1), \\[-0.2cm] 
&\Omega_A(\vec{W}) = \sum_{k=1}^{\textit{dim}} \max (0, 1 - \operatorname{Var}\left(A_k\right)),
\end{split}
\end{equation}
where \(\lambda_S\), \(\lambda_A \in \mathbb{R}^+_0\) are two multipliers. The penalty terms \(\Omega_S\) and \(\Omega_A\) are always nonnegative and equal to zero only if conditions (\ref{_fs_1}) and (\ref{_fs_2}) for all FS layer neurons are satisfied.

It can be seen that the \(\Omega_S\) penalty is a weaker form of \(\ell_1\) regularization. A neuron of the FS layer contributes to the penalty only if the sum of the absolute values of the weights entering this neuron exceeds $1$.
The contribution of an FS layer neuron to the \(\Omega_A\) penalty is positive only if the variance of its activation value is less than $1$.

The network training process minimizes the original objective function \(f\) defined by (\ref{_fs_10}) with respect to the parameter \(\vec{\tilde{\theta}}\) and balances conditions (\ref{_fs_1}) and (\ref{_fs_2}).
When the penalty \(\Omega_S\) is positive in a certain epoch of network training, to minimize it, the absolute values of the respective FS layer weights are reduced during the following epochs. According to Lemma~\ref{lemma1}, this results in a decrease in the variance of the corresponding activations and an increase in the \(\Omega_A\) penalty. Conversely, a decrease in the \(\Omega_A\) penalty leads to an increase in the absolute values of the respective FS layer weights, which can result in an increase in \(\Omega_S\). Furthermore, the prediction error must be minimized, so mainly the weights belonging to the relevant input variables should be increased, whereas the others may decrease.

\subsubsection{Differentiability of the Objective Function}
\label{subsec:nnfs_diff}
Gradient-based methods used for optimization in neural networks implicitly assume that optimized objective functions are differentiable. In examining the differentiability of the objective function \(F(\vec{\tilde{\theta}})\) described by (\ref{_fs_13}), we focus on the newly defined terms \(\Omega_S(\vec{W})\) and \(\Omega_A(\vec{W})\) penalizing the breach of conditions (\ref{_fs_1}) and (\ref{_fs_2}), respectively. 

The partial derivative of the penalty \(\Omega_S(\vec{W})\) with respect to \(w_{jk}\), where \(j \in \{1, \ldots, m\}\) and \(k \in \{1, \ldots, \textit{dim}\}\), can be expressed as follows:
\begin{equation}
\label{_fs_14}
\frac{\partial \Omega_S(\vec{W})}{\partial w_{jk}} = 
     \begin{cases}
       sgn(w_{jk}), \quad &\text{if } \sum_{i=1}^m \left|w_{ik}\right| > 1,\\
       0, \quad &\text{otherwise}.\\
     \end{cases}
\end{equation}
The function \(\Omega_S(\vec{W})\) is not mathematically differentiable with respect to \(w_{jk}\) at the points where \(\sum_{i=1}^m \left|w_{ik}\right| = 1\) and for \(w_{jk} = 0\). To make this function differentiable in its domain, we set the derivative at these points to zero.

Consider \(n\) observations of the dataset \(\vec{X}\). We denote the values of the variables \(X_1, X_2, \ldots, X_m\) of the \(i\)-th observation for \(i = 1, 2, \ldots, n\) as \(X_1^i, X_2^i, \ldots, X_m^i\). Then, the activation value of the \(k\)-th FS layer neuron for the \(i\)-th observation is
\begin{equation}
\label{_fs_15}
A_k^i = \sum_{j=1}^m w_{jk} X_j^i.
\end{equation}
Assuming that the dataset \(\vec{X}\) is standardized, the expected value of \(A_k\) is zero, and its sample variance is computed as
\begin{equation}
\label{_fs_16}
\operatorname{Var}\left(A_k\right)
= \frac{1}{n} \sum_{i=1}^n \left( A_k^i - \operatorname{E}(A_k)\right)^2 = \frac{1}{n} \sum_{i=1}^n \left( A_k^i \right)^2.
\end{equation}
After substituting (\ref{_fs_15}) into (\ref{_fs_16}), applying the chain rule, and utilizing the linearity of differentiation, the partial derivative of the variance \(\operatorname{Var}\left(A_k\right)\) with respect to \(w_{jk}\) has the following form:
\begin{equation}
\label{_fs_17}
\frac{\partial \operatorname{Var}\left(A_k\right)}{\partial w_{jk}} = 
\frac{2}{n} \sum_{i=1}^n A_k^i X_j^i =
\frac{2}{n} \sum_{i=1}^n \sum_{l=1}^m w_{lk} X_l^i X_j^i .  
\end{equation}

Now, we can compute the partial derivative of the penalty  \(\Omega_A(\vec{W})\) with respect to \(w_{jk}\) as follows:
\begin{equation}
\label{_fs_18}
\frac{\partial \Omega_A(\vec{W})}{\partial w_{jk}} = 
     \begin{cases}
       -\frac{2}{n} \sum_{i=1}^n A_k^i X_j^i,  &\text{if } \operatorname{Var}\left(A_k\right) < 1,\\
       \text{ } 0,  &\text{otherwise}.\\
     \end{cases}
\end{equation}
Technically, we solve the problem with the partial derivatives at the points where \(\operatorname{Var}\left(A_k\right) = 1\) by setting them to zero.

\subsection{Saliency Measures Based on the FS Layer}
\label{subsec:nnfs_fs}
The optimal model obtained by training the modified neural network provides the optimal weights \(\vec{\hat{W}}\) of the FS layer. Ideally, the weight matrix \(\vec{\hat{W}}\) contains \textit{dim} nonzero weights, one for each FS layer neuron (Lemma \ref{lemma2}). Then, each FS layer neuron selects exactly one important input variable corresponding to the nonzero weight. Note that several neurons can select the same variable. 

Generally, although the optimal solution is close to a sparse solution, the number of nonzero weights of \(\vec{\hat{W}}\) does not have to be small enough to unambiguously select at most \textit{dim} input variables.
Therefore, to select the required number of the most important variables, we introduce two alternative saliency measures derived from the optimal weights of the FS layer.
The first utilizes the sum of the weights, whereas the second uses the maximum weight.

\subsubsection{Sum-weight Saliency}
\label{subsec:sum-weight}
Consider the input variable \(X_j\). The simplest definition of its saliency is the sum of the absolute values of the weights leaving the input neuron corresponding to \(X_j\). We adjust this measure based on the fact that if we divide all the weights entering an FS layer neuron by a constant and simultaneously multiply all the weights leaving this neuron in the next layer by the same constant, we obtain an identical solution in terms of prediction.

Because we cannot exclude cases where the standard deviation of the activation value \(\operatorname{std}(A_k)\) in the optimal model deviates significantly from $1$, we use normalized weights when calculating the saliency; i.e., we divide each weight entering the \(k\)-th neuron of the FS layer by \(\operatorname{std}(A_k)\), for \(k \in \{1, \ldots, \textit{dim}\}\). In addition, to eliminate the dependence of the saliency value on the number of FS layer neurons, we divide the sum of the normalized weights by \textit{dim}. 
Then, we define the so-called \textit{sum-weight saliency} of the input variable \(X_j\) for the weight matrix \(\vec{\hat{W}}\) as follows:
\begin{equation}
\label{_fs_21}
\operatorname{S_{SW}}(X_j,\vec{\hat{W}}) = \frac{1}{\textit{dim}}\sum_{k=1}^{dim} \frac{\left| w_{jk} \right|}{\operatorname{std}(A_k)}. 
\end{equation}

It can be proven that the defined sum-weight saliency measure has the following property:

\begin{lemma}
\label{lemma3}
Let \(\{X_1, X_2, \ldots, X_m\}\) be the set of \(m\) variables of a standardized dataset \(\vec{X}\). If the weights \(\vec{\hat{W}}\) of the FS layer satisfy conditions (\ref{_fs_1}) and (\ref{_fs_2}), then the sum of the sum-weight saliency values of all variables of \(\vec{X}\) for the weight matrix \(\vec{\hat{W}}\) is equal to one; i.e.,
\begin{equation}
\label{_fs_23}
\sum_{j=1}^{m} \operatorname{S_{SW}}(X_j,\vec{\hat{W}}) = 1. \end{equation}
\end{lemma}

\subsubsection{Max-weight Saliency}
\label{subsec:max-weight}
The second approach is based on the assumption that each FS layer neuron selects one variable. It does not matter whether a variable has been selected one or more times. Then, the weights entering an FS layer neuron can be normalized by dividing by the sum of all these weights. The \textit{max-weight saliency} of the variable \(X_j\) for the weight matrix \(\vec{\hat{W}}\) is defined as the maximum of all normalized weights leaving the input neuron corresponding to the variable \(X_j\):
\begin{equation}
\label{_fs_22}
\operatorname{S_{MW}}(X_j,\vec{\hat{W}}) = \max_{k=1}^{\textit{dim}} \frac{\left| w_{jk} \right|}{\sum_{i=1}^m \left|w_{ik}\right|}. 
\end{equation}

\subsubsection{Comparison of Saliency Measures}
\label{subsec:compar}
It is easy to see that under conditions (\ref{_fs_1}) and (\ref{_fs_2}), the sum-weight saliency values, like the max-weight saliency values, are from the interval \([0, 1]\) for all input variables. For variables clearly selected by the FS layer neurons, the max-weight saliency values are close to $1$, whereas according to Lemma \ref{lemma3}, the sum-weight saliency values are usually smaller.

With the optimal weights of the FS layer obtained by training the modified neural network, the sum-weight saliency distinguishes in more detail the importances of individual variables, whereas the second measure based on the maximum value seems to separate the relevant variables more significantly from the irrelevant ones. The difference between these two measures is illustrated on the MNIST dataset in online supplementary material.

\section{Implementation Aspects of SNeL-FS}
\label{sec:impl}
In this section, we describe in more detail the individual steps of the proposed SNeL-FS method, focusing on the choice of network architecture, weight initialization, network training with dynamic FS layer hyperparameters, selection of the optimal model and final determination of the important input variables according to the defined saliency measures. 

\subsection{Neural Network for FS}
\label{subsec:impl_arch}
The choice of an appropriate network architecture significantly affects the results of the proposed FS method.
We always assume that the feedforward neural network used as the basis of the method is suitable for solving a given problem. Finding the ideal network architecture for a task is a specific problem that should be solved through experiments guided by monitoring the error on the validation set \cite{goodfellow2016}.

In our experiments, we usually use a simple network with one or two hidden dense layers composed of rectified linear unit (ReLU) neurons as a basis. The number of hidden neurons is affected by the number of samples available in a given dataset. For high-dimensional small-sample-size data, we try to reduce the number of network parameters due to the risk of overfitting, and therefore we use architectures with a~small number of hidden neurons.

The output layer and the type of loss function depend on the problem being solved.
In the case of classification, we utilize one sigmoid neuron for binary tasks. For multiclass tasks with \(k\) classes, we chose \(k\) softmax neurons. The cross-entropy loss function is used. In the case of regression problems, one linear neuron and the mean squared error (MSE) loss function are utilized.

\subsection{Network Training}
\label{subsec:impl_train}
The modified neural network containing the added FS layer is trained to minimize the objective function \(F(\vec{\tilde{\theta}})\) defined by (\ref{_fs_13}). The first step is to initialize the network weights. 
We assume that in all layers except the FS layer, the initial weights are randomly generated, for example, from a uniform distribution. The weights of the FS layer are initialized to a~constant value of~\(\frac{1}{2m}\), where \(m\) is the total number of input variables. Such initial weights satisfy constraint~(\ref{_fs_1}) because for each FS layer neuron, the sum of the weights entering it is~\(\frac{1}{2}\). On the other hand, they do not satisfy constraint (\ref{_fs_2}) because the variance of the new variables~\(A_k\) is at most~\(\frac{1}{4}\).
To reduce the penalty \(\Omega_A\) and increase the prediction performance, the weights increase during training, especially the weights of the relevant variables.

The amounts of penalization \(\Omega_S\) and \(\Omega_A\) are controlled by the hyperparameters \(\lambda_S\) and \(\lambda_A\), respectively. Like the other hyperparameters, these can also be chosen through common practices, such as a grid search, a random search, or manual tuning~\cite{goodfellow2016}. 
Another method is presented in~\cite{setiono1997},~\cite{watzenig2004}, where instead of fixed regularization parameters, iteratively adapted parameters are used.

We utilize the dynamic hyperparameter approach shown in~\cite{smith2017}. This methodology seems to be effective in finding a~balance between the minimization of the prediction error and the penalties for the breach of conditions (\ref{_fs_1}) and (\ref{_fs_2}).

During network training, we let the values of \(\lambda_S\) and \(\lambda_A\) cyclically vary within given ranges. We utilize a triangular window, where the values in a cycle first linearly increase and then linearly decrease.
The values of \(\lambda_S\) are evenly distributed in the range \(\left[\min\_\lambda_S, \max\_\lambda_S \right]\), and the number of values used in one half of the cycle is given by \textit{steps}\(\_\lambda_S\). Similar approach is used for \(\lambda_A\). During the network training, \(\lambda_S\) passes the number of cycles given by \textit{cycles}\(\_\lambda_S\), and for each of its values, \(\lambda_A\) passes \textit{cycles}\(\_\lambda_A\) cycles.
For each stage defined by the pair (\(\lambda_S,  \lambda_A\)), the same number of epochs is used for training and is given by the value of \textit{epochs\_per\_stage}.
The total number of epochs is determined by the 
number of cycles, steps, and epochs per stage.
The idea is illustrated in Fig.~\ref{fig:DHS}.
\begin{figure}[ht]
	\centering
	\includegraphics[width=0.95\columnwidth]{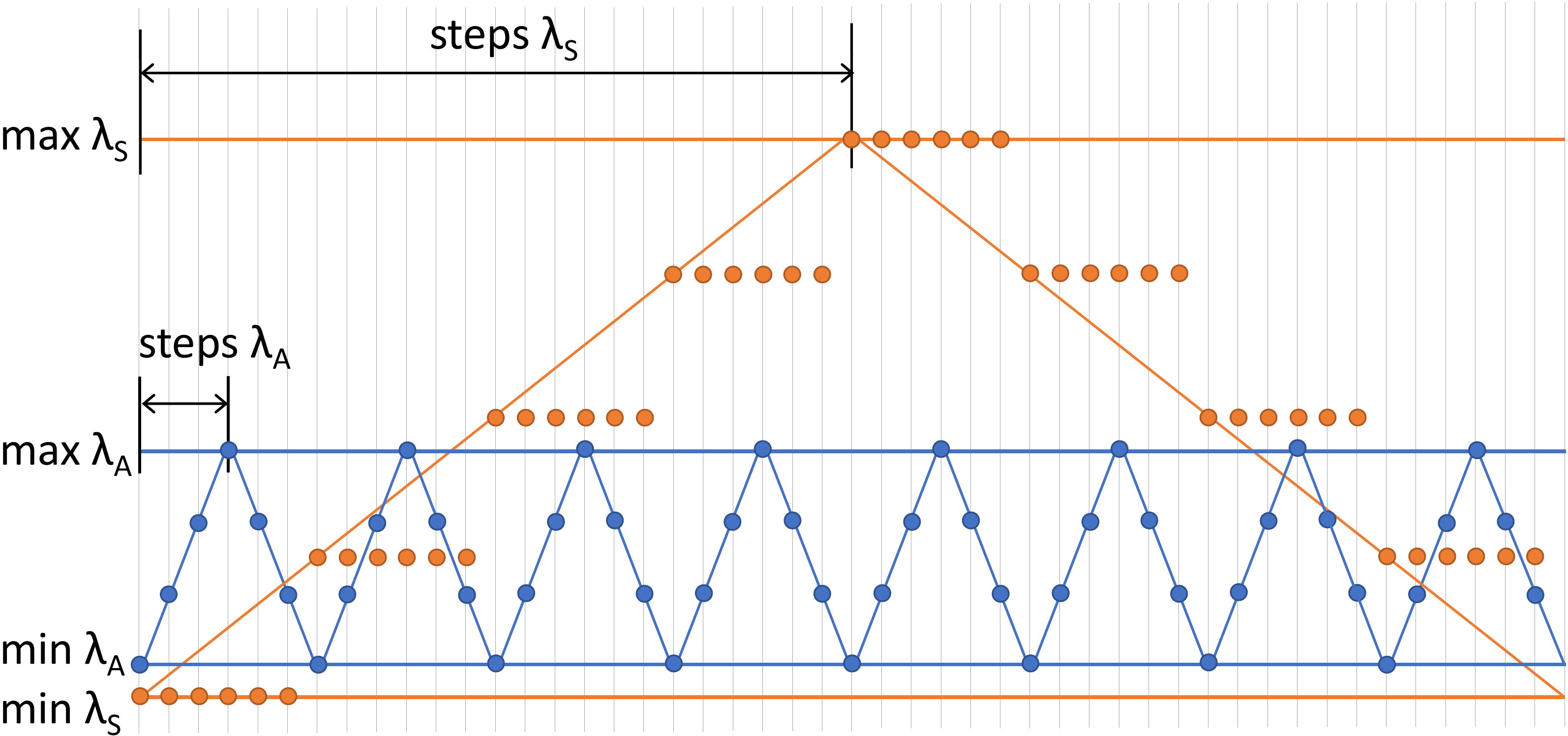}
\caption{The idea of applying cyclic hyperparameters \(\lambda_S\) (in orange) and \(\lambda_A\) (in blue) with the number of cycles \textit{cycles}\(\_\lambda_S\) = \textit{cycles}\(\_\lambda_A\) = $1$ and the numbers of steps \textit{steps}\(\_\lambda_S\) = $4$ and \textit{steps}\(\_\lambda_A\) = $3$. The number of stages is $48$.}
	\label{fig:DHS}
\end{figure}

The optimal settings of the FS layer hyperparameters depend on the properties of a particular dataset. In our experiments, we utilize the following ranges of \(\lambda_S\) and \(\lambda_A\), which can be combined: smaller \([0.001, 0.01]\), \([0.001, 0.02]\) and larger \([0.01, 0.1]\), \([0.01, 0.2]\). 
Experience shows that the best choice is to use one cycle of \(\lambda_S\) and two cycles of \(\lambda_A\) for each value of \(\lambda_S\). For high-sample datasets, we usually train $10$ epochs per stage, whereas for small-sample datasets, $1$ epoch is used per stage.

Although the FS layer weights are forced to be small, the weights of the next layers may increase during network training, and overfitting may occur, especially in the case of small-sample datasets.
To avoid this, a suitable regularization can be used in the layers following the FS layer.

The objective function \(F(\vec{\tilde{\theta}})\) in (\ref{_fs_13}) can be minimized with any optimizer. We use the Adam optimization algorithm~\cite{kingma2015} with learning rates of \(5.10^{-5}\), \(10^{-4}\), or \(10^{-3}\). For high sample-size datasets, the network can be trained with minibatches, which are also utilized for computing the variance of the FS layer activations.

\subsection{Optimal Model and FS}
\label{subsec:impl_opt}
We select the optimal model with the validation set. Before training, we split the dataset \(\vec{X}\) into a training set and a validation set with a ratio of $80:20$, and as an optimality criterion, we use the maximum prediction performance on the validation set determined by a suitable metric. If more models have the same performance, then the model with the smallest value of the objective function on the validation set is selected.

In addition, we require the model that best satisfies conditions (\ref{_fs_1}) and (\ref{_fs_2}).
Therefore, when choosing the optimal model, we take into account only models whose average penalties for the breach of conditions (\ref{_fs_1}) and (\ref{_fs_2}) per FS layer neuron do not exceed the given limit.
For the weight matrix \(\vec{W}\), the average penalty for condition (\ref{_fs_1}) is the value \({\Omega_S(\vec{W})}/{\textit{dim}}\), and for condition (\ref{_fs_2}), it is the value \({\Omega_A(\vec{W})}/{\textit{dim}}\).
We use a~value of $0.3$ for both penalties in our experiments.

The FS layer weights of the optimal model are used to calculate the sum-weight saliency or alternatively the max-weight saliency values of all features (input variables). Through the saliency values, the presented method provides a~ranking of the features according to their importance. Then, the desired number of the most important features can be selected. Another possible approach is to select the features whose saliency values are greater than a given threshold. The advantage of either of these two saliency measures depends on the dataset used and the task to be solved.

\section{Experimental Results}
\label{sec:exp}
We examined the performance of the SNeL-FS method from two important aspects. We first evaluated its ability to identify features important for determining the target variable, and then we examined its influence on the prediction accuracy. The ability to identify relevant features was tested on synthetically generated datasets, where the required output is known, and thus FS methods can be evaluated regardless of the classifier used. Experiments to evaluate the improvement in the prediction performance after applying SNeL-FS were conducted on twelve publicly available real-world high-dimensional datasets.

In the numerical experiments, we demonstrated the usability of the proposed SNeL-FS method for binary and multiclass classification problems, as well as for regression problems. As a baseline for comparison, we included three frequently used FS methods, mRMR with a mutual information (MI) criterion~\cite{peng2005}, reliefF~\cite{kononenko2003}, and f-score~\cite{weir2012}.

All presented results of the SNeL-FS method were obtained by the computationally effective TensorFlow~\cite{Abadi2015} framework.

\subsection{Experimental Results on Artificial Datasets}
\label{subsec:exp_synt}
\subsubsection{Evaluation Method}
\label{subsec:exp_synt_met}
To measure the performance of the mentioned FS methods on artificial datasets, we used the index of success (\(Suc.\)), which evaluates how well an FS method selects known relevant features (true features)~\cite{bolon-canedo2013}. The \(Suc.\) score rewards the selection of relevant features and penalizes the inclusion of irrelevant features. The index of success is defined as follows:
\begin{equation}
\label{_fs_24}
Suc. = \frac{R_s}{R_t} - \alpha\frac{I_s}{I_t},
\end{equation}
where $R_s$ is the number of selected relevant features, $I_s$ is the number of selected irrelevant features, $R_t$ is the total number of relevant features, and $I_t$ represents the total number of irrelevant features.
The parameter \(\alpha = \min \left\{ \frac{1}{2}, \frac{R_t}{I_t} \right\}\) is used to express that the selection of irrelevant features is preferred to the omission of relevant ones. If FS ranks features by their importance and all the known relevant features are selected at first, we set \(Suc. = 1\), which is the best value.

\subsubsection{Datasets}
\label{subsec:exp_synt_data}
We compared the \(Suc.\) score of the FS methods on eight synthetic datasets, four of which constitute binary classification problems and four constitute regression problems.

The Madelon dataset presents a binary classification task with numerical features. It is constructed by clusters of randomly generated points normally distributed with the standard deviation of $1$ about some vertices of the hypercube in \(N_{inf}\)-dimensional space, where \(N_{inf}\) is the number of true features. All the points of one cluster have the same class, the class $0$ is assigned to half of the clusters and the class $1$ is assigned to the other half. In our experiments, we used a \(5\)-dimensional hypercube with edges of length $2*2$ and $4$ clusters per class.

The XOR dataset represents a binary classification task with binary features. The class value is determined by the logical XOR operation of two relevant features.

The linear regression dataset represents a regression problem, where the output is generated by applying a random linear regression model with \(N_{inf}\) nonzero regressors to the well-conditioned, centered, Gaussian input with unit variance.

The Friedman dataset is a nonlinear regression dataset with independent features uniformly distributed on the interval \([0,1]\). The target variable is created according to the rule for the Friedman $1$ dataset~\cite{friedman1991} with the polynomial and sine transformations of \(N_{inf} = 5\) true features.

In the experiments, we used the Madelon, XOR, linear regression, and Friedman datasets of two variants differing in the number of samples. Small-sample datasets (denoted as Mad, XOR, Reg, Fri) consist of $200$ samples, and high-sample datasets (Mad5k, XOR5k, Reg5k, Fri5k) contain 5,000 samples. All the datasets contain $500$ features (input variables). Their basic characteristics are described in Table~\ref{tab_art_datasets}.

\begin{table}[ht]
    \footnotesize
	\caption{Characteristics of the artificial datasets.}
	\label{tab_art_datasets}       
	\centering
    \renewcommand{\arraystretch}{1.1}
	\begin{tabular}{|l|l|r|r|}
	    \hline 
		Dataset  & Acronym & Samples & Features (\(N_{inf}\)) \\
		\hline 
		Madelon & Mad & 200 & 500 (5) \\
		XOR & XOR & 200 & 500 (2)  \\
		Lin. Regression & Reg & 200 & 500 (5) \\
		Friedman & Fri & 200 & 500 (5) \\[0.05 cm]
		\hline 
		Madelon 5k & Mad5k & 5,000 & 500 (5) \\
		XOR 5k & XOR5k & 5,000 & 500 (2) \\
		Lin. Regression 5k & Reg5k & 5,000 & 500 (5) \\
		Friedman 5k & Fri5k & 5,000 & 500 (5) \\
		\hline 
	\end{tabular}
\end{table}

\subsubsection{SNeL-FS Setting}
\label{subsec:exp_synt_par}

For the high-sample datasets (denoted by $5k$), we used SNeL-FS based on the network with one hidden layer composed of $10$~ReLU neurons. For the small-sample datasets, we chose a network with two hidden layers, each with $5$~ReLU neurons, as the base. Between the input layer consisting of $500$ neurons and the first hidden layer, the FS layer with $15$ neurons was added. The number of FS layer neurons corresponds to the number of features returned by the FS method, which was determined as $3 \%$ according to~\cite{bolon-canedo2013}.

In the case of the high-sample datasets, we applied the universal range of \([0.01, 0.2]\) for the \(\lambda_S\) and \(\lambda_A\) hyperparameters with $19$ steps, and we trained $10$ epochs per stage.
For the small-sample datasets, Mad, Reg, and Fri, the same range for \(\lambda_S\) and \(\lambda_A\) was used, and only XOR required a smaller range of \([0.001, 0.02]\). For these four datasets, the number of steps was doubled to $38$, and the value of \textit{epochs\_per\_stage} was decreased to $1$.
For all datasets except XOR5k and Mad5k, we used \(\ell_1\) and \(\ell_2\) regularizations with the regularization parameters \(0.01\) in all layers following the FS layer.

As  criteria for selecting the optimal model, we utilized the maximum accuracy for classification and the minimum MSE for regression and evaluated the model on the validation set.

\subsubsection{Results}
\label{subsec:exp_synt_res}
The \(Suc.\) results of the SNeL-FS method are shown in Table~\ref{tab_art_results}. They are compared with the \(Suc.\) rates of three conventional FS methods. We present the SNeL-FS results achieved with the sum-weight saliency, and the \(Suc.\) rates determined by the max-weight saliency were similar. The best score for each dataset is emphasized in bold font.

\begin{table}[ht]
    \footnotesize
	\caption{Indices of success on eight artificial datasets.}
	\label{tab_art_results}       
	\centering
	\tabcolsep=0.15cm
    \renewcommand{\arraystretch}{1.1}
	\begin{tabular}{|l|c|c|c|c|}
	    \hline 
		Dataset & mRMR/MI & reliefF & f-score & SNeL-FS \\
		\hline 
		Mad   & 0.40          & 0.60          & 0.60          & \textbf{0.99} \\
		XOR   & 0.00          & 0.50          & 0.00          & \textbf{1.00} \\
		Reg   & 0.60          & 0.40          & 0.60          & \textbf{0.99} \\
		Fri   & 0.80          & \textbf{0.99} & 0.80          & 0.80          \\
		mean  & 0.45          & 0.62          & 0.50          & \textbf{0.95} \\
		\hline		       
		Mad5k & 0.80          & \textbf{1.00} & 0.80          & \textbf{1.00} \\
		XOR5k & 0.00          & \textbf{1.00} & 0.00          & \textbf{1.00} \\
		Reg5k & 0.99          & 0.80          & \textbf{1.00} & \textbf{1.00} \\
		Fri5k & \textbf{1.00} & \textbf{1.00} & 0.99          & \textbf{1.00} \\
		mean  & 0.70          & 0.95          & 0.70          & \textbf{1.00} \\
		\hline 
	\end{tabular}
\end{table}


For the small-sample datasets, SNeL-FS obtained an average score of \(Suc. = 0.95\) and clearly outperformed the other three methods. For the high-sample datasets, the values of the \(Suc.\) scores are generally higher than the values on the small-sample datasets because the availability of a higher number of samples allows algorithms to better recognize patterns in data. The SNeL-FS method achieved an average \(Suc. = 1.00\) on the $5k$ datasets, which means that it identified all relevant features in all these datasets as important features. Among the other methods, the reliefF method obtained a result closest to the result of the SNeL-FS method with an average \(Suc. = 0.95\). A~more detailed analysis reveals the weaknesses of the univariate f-score, especially in identifying the relevant features of the XOR datasets. Similarly, the mRMR method cannot discover the two features that together determine the target variable.
On the other hand, on the XOR5k dataset, SNeL-FS found the FS layer weights that ideally satisfy conditions (\ref{_fs_1}) and (\ref{_fs_2}) and selected exactly two relevant features.

Table~\ref{tab_art_results} shows that the SNeL-FS method was able to detect all relevant features in the examined artificial datasets except the Friedman dataset (Fri), where $80 \%$ of the relevant features were identified. These results confirm the ability of the proposed method to identify relevant features in classification and regression tasks.

\subsection{Experimental Results on Real-world Datasets}
\label{subsec:exp_real}
\subsubsection{Datasets}
\label{subsec:exp_real_data}
The influence of the proposed FS method on the prediction performance was evaluated on publicly available  microarray datasets, which represent high-dimensional classification tasks characterized by a~small number of samples, imbalanced classes, and data complexity~\cite{bolon-canedo2014}. Our experiments were performed on eight binary and four multiclass microarray datasets, whose basic properties are described in Table \ref{tab_real_datasets}. The last column shows the number of classes in the datasets and, in brackets, the number of samples in each class.

\begin{table}[ht]
    \footnotesize
	\caption{Characteristics of the real-world datasets.}
	\label{tab_real_datasets}       
	\centering
    \renewcommand{\arraystretch}{1.1}
	\begin{tabular}{|l|r|r|l|}
		\hline 
		Dataset [Source] & Samples & Features & Classes (Samples)  
		\\ \hline 
		Colon \cite{Alon:1999} & 62~ & 2,000~ & 2 (40, 22)\\
		Crohn \cite{Burczynski} & 127~ & 22,283~ & 2 (85, 42) \\
		Breast Cancer \cite{Chin} & 118~ & 22,215~ & 2 (43, 75) \\
		Breast \cite{Chowdary} & 104~ & 22,283~ & 2 (62, 42) \\
		Leukemia \cite{Golub} & 72~ & 7,129~ & 2 (47, 25)\\
		Lung \cite{Gordon} & 181~ & 12,533~ & 2 (94, 87) \\
		Prostate \cite{Singh} &  102~ & 12,600~ & 2 (52, 50) \\
		Bone Lesion \cite{tian} &  173~ & 12,625~ & 2 (36, 137) \\
		SRBCT \cite{Khan2001}  & 83~  & 2,309~  & 4 (29, 11, 18, 25)\\
        Glioma \cite{Nutt2003}    & 50~  & 12,625~ & 4 (14, 7, 14, 15) \\
        MLL \cite{Armstrong2002}    & 72~  & 12,533~ & 3 (24, 20, 28) \\
        Lung Cancer \cite{Bhattacharjee2001}  & 203~ & 12,600~ & 5 (139, 17, 6, 21, 20) \\
		\hline 
	\end{tabular}
\end{table}

\subsubsection{Evaluation Method}
\label{subsec:exp_real_met}
When evaluating the impact of the FS methods on prediction performance, several machine learning algorithms are usually used to obtain an objective assessment. We employed four well-known algorithms based on different underlying concepts, namely, the Gaussian naive Bayes (NB) classifier, the support vector classifier (SVC) with a radial basis function (RBF) kernel and a penalty parameter \(C = 1\), the random forest (RF) classifier with 1,000 base estimators and the entropy function to measure the quality of a split, and the \(k\)-nearest neighbors (kNN) classifier with \(k = 5\).

Because most of the datasets examined present classification problems with a class imbalance, we used the \(F_1\) score as a~measure of the prediction performance of the classifiers. The \(F_1\) score is defined as a harmonic mean of the precision and recall, i.e., \(F_1 = 2 . \frac{precision . recall}{precision + recall}\), and expresses the balance between them. For multiclass problems, we calculated scores for each label and found their average weighted by support, which is the number of true samples for each label.

We used \(k\)-fold stratified cross validation (CV) to validate the results, where \(k\) was set to $10$ for the binary datasets and decreased to $5$ for the multiclass datasets due to the very small number of samples in some classes. The individual feature selectors were included in the CV loop.
For each CV fold, a feature subset was obtained by an FS method with the respective training data, and then the chosen classifier was trained on the same training data with selected features. To evaluate the \(F_1\) score, the testing data for the fold were used. The final \(F_1\) score was achieved by averaging the scores over all CV folds. This FS protocol avoids biased estimations of the prediction performance~\cite{kuncheva2018}.

\subsubsection{SNeL-FS Settings}
\label{subsec:exp_real_par}
For all real-world datasets used, we utilized a~network with one hidden layer consisting of $10$ or $20$ ReLU neurons as a basis. Twenty hidden neurons were applied for the glioma and lung cancer datasets.
The FS methods selected the $30$ most important features, so the FS layer with $30$ neurons was included. This number was derived from~\cite{bolon-canedo2014}, where the ranker FS methods applied on DNA microarray datasets selecting the top $10$ and $50$ features were compared.

The FS layer hyperparameter settings were optimized for each dataset with respect to the \(F_1\) score obtained by CV and averaged over all four classifiers used.
We typically utilized the range of \([0.01, 0.1]\) for \(\lambda_S\) and \(\lambda_A\) with $18$ steps.
For the MLL dataset, the range for \(\lambda_S\) was increased to \([0.01, 0.2]\) with $19$ steps, and for the breast cancer and glioma datasets, this range was used for both \(\lambda_S\) and \(\lambda_A\). The smaller interval of \([0.001, 0.01]\) was utilized for the breast and leukemia datasets, and \([0.001, 0.02]\) was used for the lung dataset.
Additionally, the regularization parameters for \(\ell_1\) and \(\ell_2\) regularizations applied to the other layers were fine tuned. The networks were trained with $1$ epoch per stage.

For each CV fold, the respective training data were divided into training and validation parts with a ratio of \(80:20\), and the optimal model for a given fold was chosen according to the maximum \(F_1\) score on the validation set.

\subsubsection{Results}
\label{subsec:exp_real_res}
To evaluate the performance of the proposed FS method, three conventional FS methods, mRMR with the MI criterion, reliefF, and f-score, were compared. 
The \(F_1\) prediction scores after applying the FS methods are presented in Table~\ref{tab_real_results}.
For each examined dataset, we provide the individual \(F_1\) score for each classifier computed as a mean along with the standard deviation of CV as well as the average of the four applied classifiers. For SNeL-FS, the results achieved with the max-weight saliency are shown.

In the last row of Table~\ref{tab_real_results}, we present the results of the win/tie/loss (WTL) statistics that represent the number of datasets for which the average \(F_1\) score obtained after applying the SNeL-FS method is greater than, equal to or less than the average \(F_1\) score achieved by performing the respective classical method.
The results show that SNeL-FS significantly outperforms the other FS methods. With the sign test~\cite{demsar2006}, the null hypothesis that SNeL-FS and any compared method are equivalent is rejected at~an~$0.05$ level of significance.

The summary WTL statistic in the last column of the last row of Table~\ref{tab_real_results} compares the average \(F_1\) score of the SNeL-FS method with the best average scores among the three classical FS methods for each particular dataset. The table shows that SNeL-FS is almost always better than the conventional FS methods, and the two exceptions are only on the breast cancer and lung datasets. However, it can be seen that the results on the lung dataset are balanced, and all the methods achieved an average \(F_1\) score of approximately $99 \%$. The best average score, obtained by  mRMR, outperformed the average score of SNeL-FS by less than \(0.1 \%\). In contrast, SNeL-FS achieved the best score with the SVC and kNN classifiers on this dataset.
In the case of the breast cancer dataset, the results are also very balanced.
Additionally, SNeL-FS and f-score obtained the best average \(F_1\) scores of $91.57 \%$.

We can conclude that the proposed method performed better in terms of the \(F_1\) score than the other evaluated FS methods when the specified settings were used. 
This is consistent with the results obtained from the artificial data, where SNeL-FS showed a~higher rate of detection for the relevant variables.

When comparing the results in terms of the two proposed saliency measures, for the real-world datasets, better \(F_1\) scores were obtained by applying the max-weight saliency. For the synthetic datasets, the results of the \(Suc.\) score based on the sum-weight saliency and the max-weight saliency were almost equivalent.

\begin{table*}[ht]
    \footnotesize
	\caption{F1 scores on the real-world datasets. 30 features were selected.}
	\label{tab_real_results}       
	\centering
    \renewcommand{\arraystretch}{1.1}
	\begin{tabular}{|l|l|c|c|c|c|c|}
		\hline 
Dataset    & Classif. & no FS               & mRMR/MI               & reliefF               & f-score               & SNeL-FS\\    
\hline 
           & NB    & 0.5867 \(\pm\) 0.1384 & 0.7624 \(\pm\) 0.1992 & \textbf{0.7957 \(\pm\) 0.1632} & 0.7850 \(\pm\) 0.1624 & 0.7790 \(\pm\) 0.1829\\
           & SVC   & 0.5300 \(\pm\) 0.3761 & 0.6924 \(\pm\) 0.2838 & 0.7957 \(\pm\) 0.1632 & 0.7757 \(\pm\) 0.1486 & \textbf{0.8290 \(\pm\) 0.1674}\\
Colon      & RF    & 0.6867 \(\pm\) 0.2810 & \textbf{0.8090 \(\pm\) 0.1575} & 0.7757 \(\pm\) 0.1486 & 0.7900 \(\pm\) 0.1620 & 0.7757 \(\pm\) 0.1486\\
           & kNN   & 0.4267 \(\pm\) 0.3518 & 0.7257 \(\pm\) 0.1488 & 0.6957 \(\pm\) 0.2753 & 0.7757 \(\pm\) 0.1486 & \textbf{0.7790 \(\pm\) 0.1829}\\
           & avg   & 0.5575 \(\pm\) 0.2868 & 0.7474 \(\pm\) 0.1973 & 0.7657 \(\pm\) 0.1876 & 0.7816 \(\pm\) 0.1554 & \textbf{0.7907 \(\pm\) 0.1704}\\
\hline 
           & NB    & 0.7887 \(\pm\) 0.2007 & \textbf{0.9294 \(\pm\) 0.0813} & 0.7654 \(\pm\) 0.1214 & 0.9020 \(\pm\) 0.0923 & 0.9126 \(\pm\) 0.0845\\
           & SVC   & 0.8403 \(\pm\) 0.2276 & 0.9603 \(\pm\) 0.0612 & 0.8909 \(\pm\) 0.1096 & 0.9524 \(\pm\) 0.0590 & \textbf{0.9635 \(\pm\) 0.0564}\\
Crohn      & RF    & 0.8873 \(\pm\) 0.1459 & 0.9181 \(\pm\) 0.1117 & 0.8083 \(\pm\) 0.1329 & \textbf{0.9210 \(\pm\) 0.0859} & 0.9159 \(\pm\) 0.1029\\
           & kNN   & 0.8644 \(\pm\) 0.0863 & 0.9429 \(\pm\) 0.0700 & 0.8667 \(\pm\) 0.1116 & 0.9496 \(\pm\) 0.0836 & \textbf{0.9635 \(\pm\) 0.0564}\\
           & avg   & 0.8452 \(\pm\) 0.1651 & 0.9377 \(\pm\) 0.0810 & 0.8328 \(\pm\) 0.1189 & 0.9313 \(\pm\) 0.0802 & \textbf{0.9389 \(\pm\) 0.0750}\\
\hline 
           & NB    & 0.8895 \(\pm\) 0.0616 & 0.9055 \(\pm\) 0.0606 & 0.9058 \(\pm\) 0.0749 & \textbf{0.9124 \(\pm\) 0.0748} & \textbf{0.9124 \(\pm\) 0.0748}\\
           & SVC   & 0.8941 \(\pm\) 0.0567 & 0.9150 \(\pm\) 0.0594 & \textbf{0.9215 \(\pm\) 0.0521} & \textbf{0.9215 \(\pm\) 0.0521} & 0.9163 \(\pm\) 0.0525\\
Breast-    & RF    & 0.9163 \(\pm\) 0.0525 & \textbf{0.9249 \(\pm\) 0.0542} & 0.9198 \(\pm\) 0.0607 & 0.9073 \(\pm\) 0.0556 & 0.9139 \(\pm\) 0.0553\\
Cancer     & kNN   & 0.8876 \(\pm\) 0.0605 & 0.8955 \(\pm\) 0.0525 & 0.9138 \(\pm\) 0.0451 & \textbf{0.9215 \(\pm\) 0.0521} & 0.9201 \(\pm\) 0.0531\\
           & avg   & 0.8969 \(\pm\) 0.0578 & 0.9102 \(\pm\) 0.0567 & 0.9152 \(\pm\) 0.0582 & \textbf{0.9157 \(\pm\) 0.0586} & \textbf{0.9157 \(\pm\) 0.0589}\\
\hline 
           & NB    & 0.8816 \(\pm\) 0.0984 & 0.9201 \(\pm\) 0.0707 & 0.9657 \(\pm\) 0.0698 & 0.9556 \(\pm\) 0.1018 & \textbf{0.9746 \(\pm\) 0.0513}\\
           & SVC   & 0.9460 \(\pm\) 0.0667 & 0.9455 \(\pm\) 0.0716 & 0.9464 \(\pm\) 0.0864 & 0.9639 \(\pm\) 0.0786 & \textbf{0.9746 \(\pm\) 0.0513}\\
Breast     & RF    & 0.9496 \(\pm\) 0.0836 & \textbf{0.9746 \(\pm\) 0.0513} & 0.9464 \(\pm\) 0.0864 & \textbf{0.9746 \(\pm\) 0.0513} & \textbf{0.9746 \(\pm\) 0.0513}\\
           & kNN   & 0.9103 \(\pm\) 0.0989 & 0.9460 \(\pm\) 0.0667 & 0.9492 \(\pm\) 0.0630 & \textbf{0.9746 \(\pm\) 0.0513} & 0.9635 \(\pm\) 0.0564\\
           & avg   & 0.9219 \(\pm\) 0.0869 & 0.9466 \(\pm\) 0.0651 & 0.9519 \(\pm\) 0.0764 & 0.9672 \(\pm\) 0.0708 & \textbf{0.9718 \(\pm\) 0.0526}\\
\hline 
           & NB    & 0.9800 \(\pm\) 0.0600 & 0.9217 \(\pm\) 0.1234 & \textbf{0.9350 \(\pm\) 0.1001} & \textbf{0.9350 \(\pm\) 0.1001} & \textbf{0.9350 \(\pm\) 0.1001}\\
           & SVC   & 0.6233 \(\pm\) 0.3426 & 0.9514 \(\pm\) 0.0756 & 0.9514 \(\pm\) 0.0756 & 0.9057 \(\pm\) 0.1821 & \textbf{0.9657 \(\pm\) 0.0698}\\
Leukemia   & RF    & 0.9800 \(\pm\) 0.0600 & 0.9067 \(\pm\) 0.1200 & \textbf{0.9657 \(\pm\) 0.0698} & 0.9200 \(\pm\) 0.1833 & 0.9600 \(\pm\) 0.0800\\
           & kNN   & 0.6300 \(\pm\) 0.2747 & \textbf{0.9200 \(\pm\) 0.1833} & \textbf{0.9200 \(\pm\) 0.1833} & 0.8867 \(\pm\) 0.1956 & \textbf{0.9200 \(\pm\) 0.1833}\\
           & avg   & 0.8033 \(\pm\) 0.1843 & 0.9249 \(\pm\) 0.1256 & 0.9430 \(\pm\) 0.1072 & 0.9118 \(\pm\) 0.1653 & \textbf{0.9452 \(\pm\) 0.1083}\\
\hline 
           & NB    & 0.9901 \(\pm\) 0.0212 & \textbf{1.0000 \(\pm\) 0.0000} & 0.9825 \(\pm\) 0.0286 & 0.9894 \(\pm\) 0.0227 & 0.9931 \(\pm\) 0.0138\\
           & SVC   & 0.9864 \(\pm\) 0.0166 & 0.9933 \(\pm\) 0.0134 & \textbf{0.9968 \(\pm\) 0.0097} & 0.9935 \(\pm\) 0.0129 & \textbf{0.9968 \(\pm\) 0.0097}\\
Lung       & RF    & 0.9968 \(\pm\) 0.0097 & \textbf{0.9968 \(\pm\) 0.0097} & 0.9935 \(\pm\) 0.0129 & \textbf{0.9968 \(\pm\) 0.0097} & 0.9935 \(\pm\) 0.0129\\
           & kNN   & 0.9471 \(\pm\) 0.0271 & 0.9903 \(\pm\) 0.0148 & 0.9808 \(\pm\) 0.0210 & \textbf{0.9935 \(\pm\) 0.0129} & \textbf{0.9935 \(\pm\) 0.0129}\\
           & avg   & 0.9801 \(\pm\) 0.0187 & \textbf{0.9951 \(\pm\) 0.0095} & 0.9884 \(\pm\) 0.0180 & 0.9933 \(\pm\) 0.0146 & 0.9942 \(\pm\) 0.0123\\
\hline 
           & NB    & 0.7120 \(\pm\) 0.0825 & 0.9075 \(\pm\) 0.1058 & 0.9018 \(\pm\) 0.0985 & \textbf{0.9184 \(\pm\) 0.0996} & \textbf{0.9184 \(\pm\) 0.0996}\\
           & SVC   & 0.8826 \(\pm\) 0.1043 & 0.9070 \(\pm\) 0.0876 & 0.9181 \(\pm\) 0.0916 & 0.9206 \(\pm\) 0.0926 & \textbf{0.9290 \(\pm\) 0.0829}\\
Prostate   & RF    & 0.9226 \(\pm\) 0.0921 & \textbf{0.9305 \(\pm\) 0.0953} & 0.9290 \(\pm\) 0.0829 & 0.9215 \(\pm\) 0.0877 & 0.9270 \(\pm\) 0.0836\\
           & kNN   & 0.8277 \(\pm\) 0.1043 & 0.9247 \(\pm\) 0.0538 & 0.9124 \(\pm\) 0.0838 & 0.9270 \(\pm\) 0.0836 & \textbf{0.9290 \(\pm\) 0.0829}\\
           & avg   & 0.8362 \(\pm\) 0.0958 & 0.9174 \(\pm\) 0.0856 & 0.9153 \(\pm\) 0.0892 & 0.9219 \(\pm\) 0.0909 & \textbf{0.9259 \(\pm\) 0.0873}\\
\hline 
           & NB    & 0.8472 \(\pm\) 0.0624 & \textbf{0.8465 \(\pm\) 0.0561} & 0.7854 \(\pm\) 0.0696 & 0.8243 \(\pm\) 0.0603 & 0.8427 \(\pm\) 0.0507\\
           & SVC   & 0.8843 \(\pm\) 0.0115 & 0.8736 \(\pm\) 0.0406 & 0.8435 \(\pm\) 0.0539 & 0.8658 \(\pm\) 0.0454 & \textbf{0.8907 \(\pm\) 0.0324}\\
Bone-      & RF    & 0.8843 \(\pm\) 0.0115 & 0.8796 \(\pm\) 0.0474 & 0.8849 \(\pm\) 0.0246 & 0.8782 \(\pm\) 0.0311 & \textbf{0.8885 \(\pm\) 0.0198}\\
Lesion     & kNN   & 0.8794 \(\pm\) 0.0267 & 0.8720 \(\pm\) 0.0220 & 0.8786 \(\pm\) 0.0305 & 0.8782 \(\pm\) 0.0415 & \textbf{0.8830 \(\pm\) 0.0272}\\
           & avg   & 0.8738 \(\pm\) 0.0280 & 0.8679 \(\pm\) 0.0415 & 0.8481 \(\pm\) 0.0447 & 0.8616 \(\pm\) 0.0446 & \textbf{0.8762 \(\pm\) 0.0325}\\
\hline 
           & NB    & 0.9726 \(\pm\) 0.0338 & 0.9482 \(\pm\) 0.0264 & 0.7971 \(\pm\) 0.1188 & 0.9526 \(\pm\) 0.0447 & \textbf{0.9868 \(\pm\) 0.0264}\\
           & SVC   & 0.9279 \(\pm\) 0.0525 & 0.9511 \(\pm\) 0.0247 & 0.8319 \(\pm\) 0.1314 & 0.9769 \(\pm\) 0.0283 & \textbf{0.9869 \(\pm\) 0.0261}\\
SRBCT      & RF    & 1.0000 \(\pm\) 0.0000 & 0.9767 \(\pm\) 0.0285 & 0.8576 \(\pm\) 0.0988 & \textbf{0.9887 \(\pm\) 0.0226} & \textbf{0.9887 \(\pm\) 0.0226}\\
           & kNN   & 0.7938 \(\pm\) 0.0699 & 0.9399 \(\pm\) 0.0635 & 0.8170 \(\pm\) 0.0855 & \textbf{1.0000 \(\pm\) 0.0000} & 0.9756 \(\pm\) 0.0300\\
           & avg   & 0.9236 \(\pm\) 0.0390 & 0.9540 \(\pm\) 0.0358 & 0.8259 \(\pm\) 0.1086 & 0.9795 \(\pm\) 0.0239 & \textbf{0.9845 \(\pm\) 0.0263}\\
\hline 
           & NB    & 0.6863 \(\pm\) 0.0752 & 0.6235 \(\pm\) 0.1354 & 0.5620 \(\pm\) 0.1408 & 0.6655 \(\pm\) 0.1764 & \textbf{0.7174 \(\pm\) 0.1093}\\
           & SVC   & 0.6131 \(\pm\) 0.0815 & 0.6300 \(\pm\) 0.1394 & 0.6428 \(\pm\) 0.1579 & 0.6758 \(\pm\) 0.1011 & \textbf{0.7552 \(\pm\) 0.0746}\\
Glioma     & RF    & 0.6474 \(\pm\) 0.1307 & 0.7104 \(\pm\) 0.1737 & 0.6024 \(\pm\) 0.1551 & 0.6498 \(\pm\) 0.0805 & \textbf{0.7753 \(\pm\) 0.0969}\\
           & kNN   & 0.5711 \(\pm\) 0.1014 & 0.6028 \(\pm\) 0.1971 & 0.7438 \(\pm\) 0.0415 & 0.7226 \(\pm\) 0.1300 & \textbf{0.7583 \(\pm\) 0.0542}\\
           & avg   & 0.6295 \(\pm\) 0.0972 & 0.6417 \(\pm\) 0.1614 & 0.6378 \(\pm\) 0.1238 & 0.6784 \(\pm\) 0.1220 & \textbf{0.7515 \(\pm\) 0.0837}\\
\hline 
           & NB    & 0.9598 \(\pm\) 0.0328 & 0.9714 \(\pm\) 0.0352 & 0.9714 \(\pm\) 0.0352 & 0.9264 \(\pm\) 0.0492 & \textbf{0.9868 \(\pm\) 0.0264}\\
           & SVC   & 0.8913 \(\pm\) 0.0545 & 0.9430 \(\pm\) 0.0554 & 0.9321 \(\pm\) 0.0422 & 0.9587 \(\pm\) 0.0337 & \textbf{0.9598 \(\pm\) 0.0328}\\
MLL        & RF    & 0.9868 \(\pm\) 0.0264 & 0.9576 \(\pm\) 0.0348 & 0.9593 \(\pm\) 0.0333 & 0.9587 \(\pm\) 0.0337 & \textbf{0.9730 \(\pm\) 0.0330}\\
           & kNN   & 0.9072 \(\pm\) 0.0670 & \textbf{1.0000 \(\pm\) 0.0000} & 0.9450 \(\pm\) 0.0276 & 0.9431 \(\pm\) 0.0287 & 0.9550 \(\pm\) 0.0590\\
           & avg   & 0.9363 \(\pm\) 0.0452 & 0.9680 \(\pm\) 0.0313 & 0.9519 \(\pm\) 0.0346 & 0.9468 \(\pm\) 0.0363 & \textbf{0.9687 \(\pm\) 0.0378}\\
\hline 
           & NB    & 0.8847 \(\pm\) 0.0450 & 0.9166 \(\pm\) 0.0355 & 0.8076 \(\pm\) 0.0309 & 0.7946 \(\pm\) 0.0598 & \textbf{0.9262 \(\pm\) 0.0452}\\
           & SVC   & 0.8589 \(\pm\) 0.0566 & \textbf{0.9289 \(\pm\) 0.0508} & 0.7824 \(\pm\) 0.0618 & 0.7481 \(\pm\) 0.0492 & 0.9162 \(\pm\) 0.0325\\
Lung       & RF    & 0.8835 \(\pm\) 0.0682 & 0.9315 \(\pm\) 0.0276 & 0.7968 \(\pm\) 0.0705 & 0.8686 \(\pm\) 0.0528 & \textbf{0.9363 \(\pm\) 0.0235}\\
Cancer     & kNN   & 0.8767 \(\pm\) 0.0471 & 0.8996 \(\pm\) 0.0641 & 0.8293 \(\pm\) 0.0752 & 0.8572 \(\pm\) 0.0418 & \textbf{0.9163 \(\pm\) 0.0430}\\
           & avg   & 0.8759 \(\pm\) 0.0542 & 0.9191 \(\pm\) 0.0445 & 0.8040 \(\pm\) 0.0596 & 0.8171 \(\pm\) 0.0509 & \textbf{0.9238 \(\pm\) 0.0360}\\
\hline 
\textbf{WTL} & \textbf{avg} &         &       \textbf{11/0/1} &      \textbf{12/0/0} &       \textbf{11/1/0} & \textbf{10/1/1}\\

\hline 
	\end{tabular}
\end{table*}

\subsubsection{Comparison of Computational Performance}
\label{subsec:exp_real_comp}
One of the advantages of FS is that it saves computational resources. Table~\ref{tab_real_comp} compares the training times of the individual learning algorithms used in this section before and after the FS method was applied. As a training set, we chose the breast cancer dataset with 22,215 features and 118 samples represented by a matrix with 2,621,370 real numbers. After selecting $30$ features, the dataset is reduced to a new dataset with 3,540 real numbers.

\begin{table}[H]
    \footnotesize
	\caption{Comparison of computational performance.}
	\label{tab_real_comp}       
	\centering
	\tabcolsep=0.15cm
    \renewcommand{\arraystretch}{1.1}
	\begin{tabular}{|l|c|c|c|c|}
	    \hline 
		Classif. & Training time & Training time & Acceleration   \\
		algorithm &  before FS (s) &  after FS (s) &  ratio  \\
		\hline 
		NB   & 0.04635  & 0.00062    & ~74.47 \\
		SVC  & 0.47195  & 0.00099    & 476.27 \\
		RF   & 5.25366  & 0.98497    & ~~5.33 \\
		k-NN & 0.09270  & 0.00076    & 122.05 \\
		\hline 
	\end{tabular}
\end{table}

The experiments were performed on an Intel Core i5-8250U CPU with 8.00 GB RAM and \textit{scikit-learn} classifiers~\cite{scikit-learn2011}. The displayed results are the averages of 200 repetitions.

\section{Discussion}
\label{discus}
Although the experiments demonstrate the effectiveness of the SNeL-FS method, there are several open possibilities for further research. The intended use of SNeL-FS presented in this paper is as a layer of a fully connected neural network. The extension of SNeL-FS for specific types of neural networks, such as convolutional neural networks (CNNs), is not straightforward. Generally, the application of any FS for natural images where the object of interest can have an arbitrary position is not expected to be beneficial. FS can be useful in cases where the object of interest is centered in images, as is frequently the case in medical imaging, such as X-ray and computer tomography. However, also in this case, the embedding of the FS layer is not trivial since a CNN works with a specific structure of features that are processed by filters of particular size. Leaving some features can disrupt this structure, which is crucial for CNNs. Since CNNs are a popular topic, the implementation of FS in CNNs will be the object of future research.

Moreover, the original implementation of SNeL-FS can be compared to the implementation with a smoothed maximum in the penalty terms of the objective function. Other criteria for the selection of the optimal model and other optimization algorithms can also be considered. The approach with cyclic hyperparameters appears to be promising; hence, it can be useful to focus on a more effective determination of their range and an analysis of the effect against overfitting. A deeper comparison of the introduced saliency measures is another direction for further research.

\section{Conclusion}
\label{concl}
FS is an important data preprocessing strategy applied in many data mining and machine learning problems. FS methods select a subset of relevant features from original, often high-dimensional, data, thereby improving the learning performance, decreasing the computational requirements, and building better generalization models. Neural networks have the built-in ability to reduce the naturally embedded dimensions; however, the results are difficult to interpret, and further analyses are problematic. On the other hand, neural networks can be used as a basis for FS methods to preserve some original features and provide better readability and interpretability of models.

This paper presented a new supervised FS method, SNeL-FS, which uses neural networks to select important input variables. The method came from an idea inspired by the batch normalization approach. SNeL-FS constructs a special network layer, which is forced to be sparse through two constraints leading to the standardization of its activations, for variable selection. To evaluate the importance of features, two novel saliency measures based on FS layer weights were introduced.

Numerical experiments were performed on two types of data: eight artificial datasets and twelve high-dimensional real-world datasets. The results show that the proposed method is suitable for binary and multiclass classification problems as well as for regression problems. On the synthetically generated datasets, the SNeL-FS method effectively identified the relevant features. The results on the real-world datasets proved that the method is able to efficiently reduce dimensionality and achieve the best predictive performance in terms of the average \(F_1\) score. All the experiments were executed by effective GPU implementation with the TensorFlow framework.

\section*{Acknowledgment}

This work was supported by the Scientific Grant Agency of the Ministry of Education, Science, Research and Sport of the Slovak Republic and the Slovak Academy of Sciences under contract VEGA 1/0327/20.

\ifCLASSOPTIONcaptionsoff
\newpage
\fi

\bibliographystyle{IEEEtran}
\bibliography{pb_nnet_fs.bib}







\end{document}